\documentclass[10pt,a4paper]{article}

\usepackage[latin1]{inputenc}
\usepackage{verbatim}
\usepackage[T1]{fontenc}
\usepackage{a4wide}
\usepackage{palatino}
\usepackage{multicol}
\usepackage{tikz}
\usetikzlibrary{arrows,calc,fit,patterns,plotmarks,shapes.geometric,shapes.misc,
   shapes.symbols,shapes.arrows,%
   shapes.callouts,shapes.multipart,shapes.gates.logic.US,shapes.gates.logic.IEC,%
   er,automata,backgrounds,chains,topaths,%
   trees,petri,mindmap,matrix,calendar,folding,fadings,through,positioning,scopes,
   decorations.fractals,%
   decorations.shapes,decorations.text,decorations.pathmorphing,
   decorations.pathreplacing,%
   decorations.footprints,decorations.markings,shadows}

\usepackage{amsmath}
\usepackage{amssymb}
\usepackage{amsthm}
\usepackage{algorithm}
\usepackage{algorithmic}
\usepackage{graphicx}
\newtheorem{theorem}{Theorem}
\newtheorem{lemma}{Lemma}

\newtheorem{corollary}{Corollary}

\newtheorem{remark}{Remark}
\newtheorem{definition}{Definition}

\newcommand{\CSP}[1]{\ensuremath{\operatorname{CSP}(#1)}}
\newcommand{\cspneg}[1]{\ensuremath{\CSP{\overline{#1}}}}
\newcommand{\tuple}[1]{\ensuremath{\left\langle {#1} \right\rangle}}
\newenvironment{notation}{\vspace{0.3cm} \noindent {\bf Notation:\ }}{\vspace{0.3cm}}

\title{A Dichotomy for 2-Constraint Forbidden
CSP Patterns\thanks{supported by ANR Project ANR-10-BLAN-0210.}}
\date{}
\author{Martin C. Cooper \ \ \ \ \  Guillaume Escamocher\\
IRIT, University of Toulouse, France (cooper@irit.fr)}

\begin{document}

\maketitle

\begin{abstract}
Although the CSP (constraint satisfaction problem) is NP-complete,
even in the case when all constraints are binary, certain classes of
instances are tractable. We study classes of instances defined by
excluding subproblems. This approach has recently led to the
discovery of novel tractable classes. The complete characterisation
of all tractable classes defined by forbidding patterns (where a
pattern is simply a compact representation of a set of subproblems)
is a challenging problem. We demonstrate a dichotomy in the case of
forbidden patterns consisting of either one or two constraints. This
has allowed us to discover new tractable classes including, for
example, a novel generalisation of 2SAT.
\end{abstract}

\section{Introduction}
In this paper we study the generic combinatorial problem known as the
binary constraint satisfaction problem (CSP) in which the aim is to
determine the existence of an assignment of values to $n$ variables
such that a set of constraints on pairs of variables are
simultaneously satisfied. The generic nature of the CSP has led to
diverse applications, notably in the fields of Artificial
Intelligence and Operations Research.

A fundamental research question in complexity theory is the
identification of tractable subproblems of NP-complete problems.
Classical approaches have consisted in identifying types of
constraints which imply the existence of a polynomial-time algorithm.
Among the most well-known examples, we can cite linear constraints
and Horn clauses. In an orthogonal approach, restrictions are placed
solely on the (hyper)graph of constraint scopes. In some cases,
dichotomies have even been proved characterising all tractable
classes definable by placing restrictions either on the constraint
relations~\cite{Bulatov2005:classifying,Bulatov2006:dichotomy} or on
the (hyper)graph of constraint
scopes~\cite{Grohe2007:complexity,Marx2010,Marx2010:tractable}.

Recently, a new avenue of research has been investigated: the
identification of tractable classes of CSP instances defined by
forbidding a specific (set of) subproblems. Novel tractable classes
have been discovered by forbidding simple 3-variable
subproblems~\cite{btp,cz11:aij}. This paper presents an essential
first step towards the identification of all such tractable classes,
namely a dichotomy for the special case of forbidden 2-constraint
subproblems.

We first define the notion of a CSP pattern. A pattern can be seen as
a generalisation of a binary CSP instance; it represents a \emph{set}
of subproblems by leaving the consistency of some tuples undefined.
We use the term \emph{point} to denote an assignment of a value to a
variable, i.e. a pair $a = \tuple{v,d}$ where $d$ is in the domain of
variable $v$. A pattern is a graph in which vertices correspond to
points and both vertices and edges are labelled. The label of a
vertex corresponding to an assignment \tuple{v,d} is simply the
variable $v$ and the label of an edge between two vertices describes
the compatibility of the pair of assignments corresponding to the
pair of vertices.

\begin{definition}
A \emph{pattern} is a quintuplet \tuple{V,A,var,E,cpt} comprising:
\begin{itemize}
\item a set $V$ of variables,
\item a set $A$ of points (assignments),
\item a variable function $var : A \to V$,
\item a set $E \subseteq \binom{A}{2}$ of edges
(unordered pairs of elements of $A$) such that $\{a,b\} \in E
\Rightarrow var(a) \neq var(b)$,
\item a Boolean-valued compatibility function $cpt : E \to \{F,T\}$, where
for notational simplicity, we write $cpt(a,b)$ instead of $cpt(\{a,b\})$.
\end{itemize}
\end{definition}

\begin{definition}
A \emph{binary CSP instance} is a pattern \tuple{V,A,var,E,cpt} such
that $E = \{(a,b) \mid var(a) \neq var(b)\}$ (i.e. the compatibility
of each pair of assignments to distinct variables is specified by the
compatibility function). The question corresponding to the instance
is: does there exist a consistent set of assignments to all the
variables, that is a \emph{solution} $\overline{A} \subseteq A$ such
that $|\overline{A}| = |V|$, ($\forall a,b \in \overline{A}$, $var(a)
\neq var(b)$) and ($\forall e \in \binom{\overline{A}}{2}$,
$cpt(e) = T$)?
\end{definition}

For a pattern $P = \tuple{V,A,var,E,cpt}$ and a variable $v \in V$,
we use $A_v$ to denote the set of assignments $\{a \in A \mid var(a)
= v\}$. The \emph{constraint} on variables $v_1,v_2 \in V$ is the
pattern \tuple{\{v_1,v_2\},A_{12},var|_{A_{12}},E_{12},cpt|_{E_{12}}}
where $A_{12} = A_{v_1} \cup A_{v_2}$ and $E_{12} = \{ \{a,b\} \mid a
\in A_{v_1}, b \in A_{v_2} \}$. If $cpt(a,b)=T$ then the two
assignments (points) $a,b$ are \emph{compatible} and $\{a,b\}$ is a
\emph{compatibility edge}; if $cpt(a,b)=F$ then the two assignments
$a,b$ are \emph{incompatible} and $\{a,b\}$ is a
\emph{incompatibility edge}. In a pattern, the compatibility of a
pair of points $a,b$ such that $var(a) \neq var(b)$ and $(a,b) \notin
E$ is \emph{undefined}. A pattern can be viewed as a means of
representing the set of all instances obtained by arbitrarily
specifying the compatibility of such pairs. Two patterns $P$ and $Q$
are \emph{isomorphic} if they are identical except for a possible
renaming of variables and assignments.

In a CSP instance \tuple{V,A,var,E,cpt}, we call the set $\{d \mid
\tuple{v,d} \in A\}$ of values that can be assigned to variable $v$
the \emph{domain} of $v$ and the set $\{(a,b) \in A_{v_1} \times
A_{v_2} \mid cpt(a,b) = T\}$ of compatible pairs of values that can
be assigned to two variables $v_1,v_2 \in V$ the \emph{constraint
relation} on $v_1,v_2$. The constraint between variables $v_1$ and
$v_2$ in an instance is \emph{non-trivial} if there is at least one
incompatible pair of assignments, i.e. $a \in A_{v_1}$ and $b \in
A_{v_2}$ such that $cpt(a,b)=F$. The \emph{constraint graph} of an
instance \tuple{V,A,var,E,cpt} is \tuple{V,H}, where $H$ is the set
of pairs of variables $v_1,v_2 \in V$ such that the constraint on
$v_1,v_2$ is non-trivial.

\begin{definition} \label{def:occurs}
We say that a pattern $P$ \emph{occurs} in a pattern $P'$ (or that
$P'$ contains $P$) if $P'$ is isomorphic to a pattern $Q$ in the
transitive closure of the following two operations (extension and
merging) applied to $P$:
\begin{description}
\item[extension] $P$ is a sub-pattern of $Q$ (and $Q$ an extension of $P$):
if $P =\tuple{V_P,A_P,var_P,E_P,cpt_P}$ and $Q
=\tuple{V_Q,A_Q,var_Q,E_Q,cpt_Q}$, then $V_P \subseteq V_Q$, $A_P
\subseteq A_Q$, $var_P=var_Q|_{P}$, $E_P \subseteq E_Q$,
$cpt_P=cpt_Q|_{E_P}$. Example:

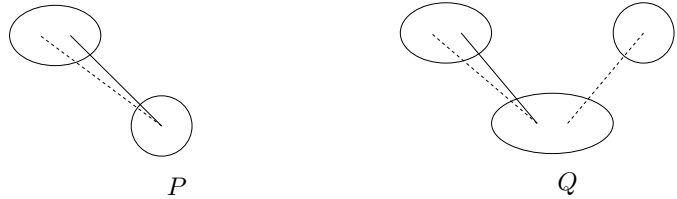
\hfill
\scalebox{0.4}{\begin{tikzpicture} \node[scale=2.4] at
(0.5,-2) {$P$}; \draw[thick,color=black,style=dashed] (0,0)--(-4,3);
\draw[thick,color=black] (-3,3)--(0,0);
\draw (0,0) circle (1); \draw (-3.5,3) circle (1.5 and 1);
\end{tikzpicture}}
\hfill
\scalebox{0.4}{\begin{tikzpicture} \node[scale=2.4] at
(0.5,-2) {$Q$}; \draw[thick,color=black,style=dashed]
(-0.5,0)--(-4,3) (0.5,0)--(3,3); \draw[thick,color=black]
(-3,3)--(-0.5,0);
\draw (0,0) circle (2 and 1); \draw (-3.5,3) circle (1.5 and 1);
\draw (3,3) circle (1);
\end{tikzpicture}}
\hfill {}

\item[merging] Merging two points in $P$ transforms $P$ into $Q$:
if $P =\tuple{V_P,A_P,var_P,E_P,cpt_P}$ and $Q
=\tuple{V_Q,A_Q,var_Q,E_Q,cpt_Q}$, then $\exists a,b \in A_P$ such
that $var_P(a)=var_P(b)$ and $\forall c \in A_P$ such that $\{a,c\},
\{b,c\} \in E_P$, $cpt_P(a,c) = cpt_P(b,c)$. Furthermore,
$V_P=V_Q$, $A_Q=A_P \setminus \{b\}$, $var_Q=var_P|_{A_Q}$, $E_Q =
E_P \cup \{\{a,x\} \mid \{b,x\} \in E_P\}$ and $cpt_Q(a,x) =
cpt_Q(b,x)$ if $\{b,x\} \in E_P$, $cpt_Q(e)=cpt_P(e)$ otherwise.
Example:

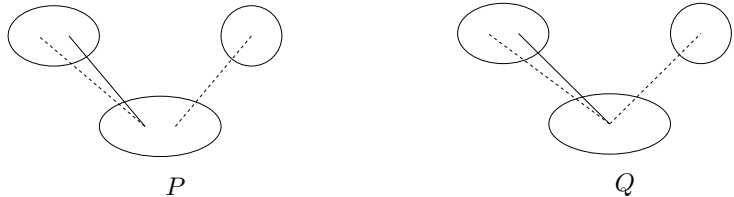
\hfill \scalebox{0.4}{\begin{tikzpicture} \node[scale=2.4] at
(0.5,-2) {$P$}; \draw[thick,color=black,style=dashed]
(-0.5,0)--(-4,3) (0.5,0)--(3,3); \draw[thick,color=black]
(-3,3)--(-0.5,0);
\draw (0,0) circle (2 and 1); \draw (-3.5,3) circle (1.5 and 1);
\draw (3,3) circle (1);
\end{tikzpicture}}
\hfill \scalebox{0.4}{\begin{tikzpicture} \node[scale=2.4] at
(0.5,-2) {$Q$}; \draw[thick,color=black,style=dashed] (0,0)--(-4,3)
(0,0)--(3,3); \draw[thick,color=black] (-3,3)--(0,0);
\draw (0,0) circle (2 and 1); \draw (-3.5,3) circle (1.5 and 1);
\draw (3,3) circle (1);
\end{tikzpicture}}
\hfill {}

\end{description}
\end{definition}

\begin{notation}
Let $P$ be a CSP pattern. We use $\cspneg{P}$ to denote the set of
binary CSP instances $Q$ in which $P$ does not occur.
\end{notation}


\begin{definition}
A pattern $P$ is \emph{intractable} if $\cspneg{P}$ is NP-complete.
It is \emph{tractable} if there is a polynomial-time algorithm to
solve $\cspneg{P}$.
\end{definition}

In this paper we characterise all tractable two-constraint patterns.
It is worth observing that, in a class of CSP instances defined by
forbidding a pattern, there is no bound on the size of domains.
Recall, however, that CSP instances have finite domains since the set
of all possible assignments is assumed to be given in extension as
part of the input.

Clearly, all classes of CSP instances $\cspneg{P}$ defined by
forbidding a pattern are hereditary: $I \in \cspneg{P}$ and $I'
\subseteq I$ (in the sense that $I$ is an extension of $I'$,
according to Definition~\ref{def:occurs}) together imply that $I' \in
\cspneg{P}$. Furthermore, if $I \in \cspneg{P}$ and $I'$ is
isomorphic to $I$, then $I' \in \cspneg{P}$. Forbidding a pattern
therefore only allows us to define hereditary classes closed under
arbitrary permutations of variable domains.

\section{Preprocessing operations on CSP instances}

This section describes polynomial-time simplification operations on
CSP instances. Assuming that these operations have been applied
facilitates the proof of tractability of many patterns.

Let \tuple{V,A,var,E,cpt} be a CSP instance. If for some variable
$v$, $A_v$ is a singleton $\{a\}$, then the \emph{elimination of a
single-valued variable} corresponds to making the assignment $a$ and
consists of eliminating $v$ from $V$ and eliminating $a$ from $A$ as
well as all assignments $b$ which are incompatible with $a$.

Given a CSP instance \tuple{V,A,var,E,cpt}, \emph{arc consistency}
consists in eliminating from $A$ all assignments $a$ for which there
is some variable $v \neq var(a)$ in $V$ such that $\forall b \in
A_v$, $cpt(a,b) = F$~\cite{Bessiere:AC}.

Given a CSP instance \tuple{V,A,var,E,cpt}, if for $var(a)=var(b)$
and for all variables $v \neq var(a)$, $\forall c \in A_v$, $cpt(a,c)
=T$ $\Rightarrow$ $cpt(b,c)=T$, then we can eliminate $a$ from $A$ by
\emph{neighbourhood substitution}, since in any solution in which $a$
appears, we can replace $a$ by $b$~\cite{neighone}. Establishing arc
consistency and eliminating single-valued variables until convergence
produces a unique result, and the result of applying neighbourhood
substitution operations until convergence is unique modulo
isomorphism~\cite{neightwo}. None of these three operations
when applied to an instance in \cspneg{P} can introduce the forbidden
pattern $P$.

We now consider two new simplification operations. They are
simplification operations that can be applied to certain CSP
instances. We can always perform the fusion of two variables $v_1$,
$v_2$ in a CSP instance into a single variable $v$ whose set of
assignments is the cartesian product of the sets of assignments to
$v_1$ and to $v_2$. Under certain conditions, we do not need to keep
all elements of this cartesian product and, indeed, the total number
of assignments actually decreases.

\begin{definition}
Consider a CSP instance \tuple{V,A,var,E,cpt} with $v_{1},v_{2} \in
V$. Suppose that there is a \emph{fusion function} $f : A_{v_1} \to
A_{v_2}$, such that $\forall u \in A_{v_1}$, whenever $u$ is in a
solution $S$, there is a solution $S'$ containing both $u$ and
$f(u)$. Then we can perform the \emph{simple fusion} of $v_2$ and
$v_1$ to create a new \emph{fused} variable $v$. The resulting
instance is \tuple{V',A',var',E',cpt'} defined by $V' = (V \setminus
\{v_1,v_2\}) \cup \{v\}$, $A' = A \setminus A_{v_2}$, $var'(u) =
var(u)$ for all $u \in A' \setminus A_{v_1}$ and $var'(u) = v$ for
all $u \in A_{v_1}$, $E' = \{(p,q) \in \binom{A'}{2} \mid var'(p)
\neq var'(q)\}$, $cpt'(p,q) = cpt(p,q)$ if $p,q \in A' \setminus
A_{v_1}$, $cpt'(u,q) = cpt(u,q) \wedge cpt(f(u),q)$ for all $u \in
A_{v_1}$ and all $q \in A' \setminus A_{v_1}$.
\end{definition}

\begin{definition}
Consider a CSP instance \tuple{V,A,var,E,cpt} with $v_{1},v_{2} \in
V$ and a \emph{hinge value} $a \in A_{v_1}$. Suppose that there is a
\emph{fusion function} $f : A_{v_1} \setminus \{a\} \to A_{v_2}$,
such that $\forall u \in A_{v_1} \setminus \{a\}$, whenever $u$ is in
a solution $S$, there is a solution $S'$ containing both $u$ and
$f(u)$. Then we can perform the \emph{complex fusion} of $v_2$ and
$v_1$ to create a new \emph{fused} variable $v$. The resulting
instance is \tuple{V',A',var',E',cpt'} defined by $V' = (V \setminus
\{v_1,v_2\}) \cup \{v\}$, $A' = A \setminus \{a\}$, $var'(u) =
var(u)$ for all $u \in A' \setminus (A_{v_1} \cup A_{v_2})$ and
$var'(u) = v$ for all $u \in (A_{v_1} \setminus \{a\}) \cup A_{v_2}$,
$E' = \{(p,q) \in \binom{A'}{2} \mid var'(p) \neq var'(q)\}$,
$cpt'(p,q) = cpt(p,q)$ if $p,q \in A' \setminus (A_{v_1} \cup
A_{v_2})$, $cpt'(u,q) = cpt(u,q) \wedge cpt(f(u),q)$ for all $u \in
A_{v_1} \setminus \{a\}$ and all $q \in A' \setminus (A_{v_1} \cup
A_{v_2})$, $cpt'(p,q) = cpt(a,q) \wedge cpt(p,q)$ for all $p \in
A_{v_2}$ and all $q \in A' \setminus (A_{v_1} \cup A_{v_2})$.
\end{definition}

\begin{lemma} \label{lem:fusion}
If $I$ is a CSP instance and $I'$ the result of a (simple or complex)
fusion of two variables in $I$, then $I'$ is solvable iff $I$ is
solvable.
\end{lemma}

\begin{proof}
We give the proof only for the case of a complex fusion, since a
simple fusion can be considered as a special case. Among the
assignments in the cartesian product of $A_{v_1}$ and $A_{v_2}$, it
is sufficient, in order to preserve solvability, to keep only those
of the form $(a,q)$ where $q \in A_{v_2}$ or of the form $(u,f(u))$
where $u \in A_{v_1} \setminus \{a\}$. To complete the proof, it
suffices to observe that in $A'$ we use $q \in A_{v_2}$ to represent
the pair of assignments $(a,q)$ and $u \in A_{v_1} \setminus \{a\}$
to represent $(u,f(u))$.
\end{proof}

Fusion preserves solvability and the total number of assignments
decreases by at least 1 (in fact, by $|A_{v_2}|$ in the case of a
simple fusion). However, when solving instances $I \in \cspneg{P}$,
for some pattern $P$, a fusion operation will only be useful if it
does not introduce the forbidden pattern $P$.

\section{Reduction}

In a pattern $P$, a point $a$ which is linked by a single
compatibility edge to the rest of $P$ is known as a {\em dangling
point}. If an arc consistent instance $I$ does not contain the
pattern $P$ then it does not contain the pattern $P'$ which is
equivalent to $P$ in which the dangling point $a$ and the
corresponding compatibility edge have been deleted. Thus, since arc
consistency is a polynomial-time operation which cannot introduce a
forbidden pattern, to decide tractability we only need consider
patterns without dangling points.

\begin{definition}
We say that a pattern $P$ can be \emph{reduced} to a pattern $Q$, and
that $Q$ is a \emph{reduction} of $P$, if $Q$ is in the transitive
closure of the three operations extension, merging and dp-elimination
applied to $P$, where dp-elimination is the following operation:

\begin{description}
\item[dp-elimination] Eliminating a dangling point and its corresponding compatibility
edge from $P$ transforms $P$ into $Q$.
Example:

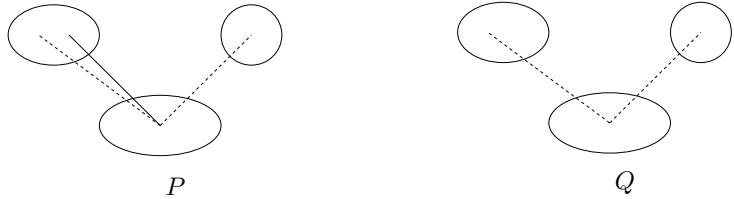
\hfill \scalebox{0.4}{\begin{tikzpicture} \node[scale=2.4] at
(0.5,-2) {$P$}; \draw[thick,color=black,style=dashed] (0,0)--(-4,3)
(0,0)--(3,3); \draw[thick,color=black] (-3,3)--(0,0);
\draw (0,0) circle (2 and 1);
\draw (-3.5,3) circle (1.5 and 1);
\draw (3,3) circle (1);
\end{tikzpicture}}
\hfill \scalebox{0.4}{\begin{tikzpicture} \node[scale=2.4] at
(0.5,-2) {$Q$}; \draw[thick,color=black,style=dashed] (0,0)--(-4,3)
(0,0)--(3,3);
\draw (0,0) circle (2 and 1);
\draw (-3.5,3) circle (1.5 and 1);
\draw (3,3) circle (1);
\end{tikzpicture}}
\hfill {}

\end{description}
\end{definition}

\begin{lemma}\label{recdef}
Let $P$ and $Q$ be two patterns, such that $P$ can be reduced to $Q$.
Let $I$ be a CSP instance satisfying arc consistency. If $Q$ occurs
in $I$, then $P$ also occurs in $I$.
\end{lemma}

\begin{proof}
By definition, reduction is a transitive relation. Therefore, by
induction, it suffices to prove the result for each of the individual
operations: extension, merging and dp-elimination. We suppose $Q$
occurs in $I$. If $Q$ is an extension of $P$, then $P$ is a
sub-pattern of $Q$ and the result is immediate. If merging two points
$a$ and $b$ in $P$ transforms it into $Q$, then $P$ actually covers
two different patterns: the one where $a$ and $b$ are different
points, and the one where $a$ and $b$ are the same point. The latter
pattern is $Q$. So the set of instances containing $Q$ is a subset of
the set of instances containing (at least one of the two versions of)
$P$ and we have the result. If adding
a dangling point and its corresponding compatibility edge to $Q$
transforms it into $P$, then since $I$ satisfies arc consistency $P$
also occurs in $I$.
\end{proof}

The following corollary follows immediately from the fact that arc
consistency can be established in polynomial time.

\begin{corollary} \label{cor:(in)tractable}
Let $P$ and $Q$ be two patterns, such that $P$ can be reduced to $Q$.
Then
\begin{itemize}
\item If $Q$ is tractable, then $P$ is tractable.
\item If $P$ is intractable, then $Q$ is intractable.
\end{itemize}
\end{corollary}

It follows that we only need to study those patterns that cannot be
reduced to a known tractable pattern and that are not the reduction
of a known intractable pattern.

\section{One-constraint patterns}

In this section we prove a dichotomy for patterns composed of a
single constraint. We also prove some results concerning 1-constraint
patterns that are essential for the proof of the 2-constraint
dichotomy given in Section~\ref{sec:2-constraint}.

\begin{lemma}\label{twored}
Let $P$ be a pattern such that a constraint in $P$ contains two
distinct incompatibility edges that cannot be merged. Then $P$ is
intractable.
\end{lemma}

\begin{proof}
Let $P$ be a pattern such that a constraint in $P$ contains two
non-mergeable incompatibility edges. Let SAT1 be the set of SAT
instances with at most one occurrence of each variable in each
clause. SAT1 is trivially equivalent to SAT which is well known to be
NP-complete~\cite{Cook}. It suffices to give a polynomial reduction
from SAT1 to \cspneg{P}. We suppose that we have a SAT1 instance
$I=\{V,S\}$ with $V$ a set of variables $\{v_1,v_2,\dots,v_n\}$ and
$S$ a set of clauses $\{C_1,C_2,\dots,C_k\}$ such that each clause
$C_i$ is a disjunction of $c_i$ literals $l_i^1\vee\dots\vee
l_i^{c_i}$. We create the following CSP instance $I'$:
\begin{itemize}
\item $n+k$ variables $v'_1,\dots,v'_{n+k}$.
\item $\forall v'_i$ with $1\leq i\leq n$, two points "$v_i$" and
"$\overline{v_i}$" in $A_{v'_i}$.
\item $\forall v'_i$ with $n+1\leq i\leq n+k$, $c_{i-n}$ points
$l_{i-n}^1,\dots,l_{i-n}^{c_{i-n}}$ in $A_{v'_i}$.
\item $\forall 1\leq i\leq k$, $\forall 1\leq j\leq c_i$,
an incompatibility edge between the point $l_i^j\in A_{v'_{n+i}}$ and
the occurrence in $A_{v'_1},\dots,A_{v'_n}$ of the literal
$\overline{l_i^j}$.
\end{itemize}
By construction, $I'$ has a solution if and only if $I$ has a
solution. Furthermore, each time an incompatibility edge occurs in a
constraint $C$, this constraint $C$ is between a CSP variable $v_i'$
representing the SAT1 variable $v_i$ and another CSP variable
$v'_{n+j}$ representing the SAT1 clause $C_j$. Since $v_i$ occurs at
most once in $C_j$, then there is only one incompatibility edge in
$C$. So $I'$ does not contain the pattern $P$. So we have reduced
SAT1 to \cspneg{P}.
\end{proof}

\begin{definition}
Given a pattern $P=\tuple{V,A,var,E,cpt}$, a variable $v \in V$, and
a point $a \in A_v$, we say that $a$ is \emph{explicitly compatible}
(respectively \emph{explicitly incompatible}) if there is a point $b
\in A$ such that $a$ is compatible with $b$ (respectively such that
$a$ is incompatible with $b$).
\end{definition}

\begin{lemma}\label{onegreenpoint}
Let $P$ be a non-mergeable pattern. Then for every variable $v$ in
$P$, there is at most one point in $A_v$ which is not explicitly
incompatible.
\end{lemma}

\begin{proof}
We suppose we have a pattern $P$ such that there are two points $a$
and $b$ with $var(a)=var(b)$ such that neither $a$ nor $b$ is
explicitly incompatible. So no point in the pattern is incompatible
with either $a$ or $b$. Hence, we can merge $a$ and $b$, which is a contradiction.
\end{proof}

Let $Z$ be the pattern on two variables $v$ and $v'$,
with points $a,b \in A_v$ and points $c,d \in A_{v'}$ such that
$a$ is compatible with both $c$ and $d$, $b$ is compatible with $c$
and incompatible with $d$.

\begin{lemma}\label{zed}
$Z$ is intractable.
\end{lemma}

\begin{proof}
Since 3-{\sc colouring} is NP-complete~\cite{coloring}, it suffices
to give a polynomial reduction from 3-{\sc colouring} to
CSP($\overline{Z}$), the set of CSP instances in which the pattern
$Z$ does not occur.

Define the relation $R_{s,t} \subseteq \{1,2,3\}^2$ by
\[R_{s,t} \ = \ \{\langle u,v \rangle | (u=s \wedge v=t) \vee (u \neq s \wedge v \neq t)\}
\]
It is easy to verify that $R_{s,t}$ does not contain the pattern $Z$.
Consider the 5-variable gadget with variables $v_i,v_j,u_1,u_2,u_3$, each
with domain $\{1,2,3\}$, and with constraints $R_{k,k}$
on variables $(v_i,u_k)$ ($k=1,2,3$) and constraints $R_{1+(k \ {\rm mod} \ 3),k}$
on variables $(u_k,v_j)$ ($k=1,2,3$). The joint effect of these six
constraints is simply to impose the constraint $v_i \neq v_j$.
Any instance $\langle V,E \rangle$
of 3-{\sc colouring}, with $V=\{1,\ldots,n\}$, can be reduced to an instance of
CSP($\overline{Z}$) with variables $v_1,\ldots,v_n$
by placing a copy of this gadget between every pair of variables $(v_i,v_j)$ such that
$\{i,j\} \in E$. This reduction is clearly polynomial.
\end{proof}

Let $1I$ be the pattern on two variables $v$ and $v'$ with
points $a\in A_v$ and $b\in A_{v'}$ such that
$a$ and $b$ are incompatible.

\begin{lemma} \label{lem:1_constraint_dichotomy}
Let $P$ be a pattern on one constraint. Then either $P$ is reducible
to the trivial tractable pattern $1I$, and thus is tractable, or $P$
is intractable.
\end{lemma}

\begin{proof}
Let $P$ be a pattern on one constraint between two variables $v$ and
$v'$. From Lemma~\ref{twored}, we know that if $P$ has two distinct
incompatibility edges, then $P$ is intractable. If there is no
incompatibility edge at all in $P$, then $P$ is reducible by merging
and/or dp-elimination to the empty pattern, which is itself reducible
by sub-pattern to $1I$. We suppose there is exactly one
incompatibility edge in $P$. We label $a\in A_{v}$ and $b\in A_{v'}$
the points defining that edge. From Lemma~\ref{onegreenpoint}, we
know that we only need to consider at most one other point $c\neq a$
in $A_{v}$ and at most one other point $d\neq b$ in $A_{v'}$. If all
three edges $\{a,d\}$, $\{c,b\}$ and $\{c,d\}$ are compatibility
edges, then $P$ is intractable from Lemma~\ref{zed}. If only two or
less of these edges are compatibility edges, then $P$ is reducible by
merging and/or dp-elimination to $1I$. So we have the lemma.
\end{proof}

\begin{lemma} \label{lem:separate_patterns}
Let $P$ be a pattern composed of two separate one-constraint
patterns: $P_1$ on variables $v_0,v_1$ and $P_2$ on variables
$v_2,v_3$, where all four variables are distinct. Then
\begin{enumerate}
\item If either $P_1$ or $P_2$ is intractable, then $P$ is intractable too.
\item If both $P_1$ and $P_2$ are tractable, then $P$ is tractable.
\end{enumerate}
\end{lemma}

\begin{proof}
\begin{enumerate}
\item $P_1$ and $P_2$ are sub-patterns of $P$, so they are both reducible to $P$.
So if one of them is intractable, then $P$ is intractable too, by
Corollary~\ref{cor:(in)tractable}.
\item Suppose that both $P_1$ and $P_2$ are tractable. So there are two polynomial
algorithms $A_1$ and $A_2$ which solve \cspneg{P_1} and \cspneg{P_2},
respectively. Let $I$ be a CSP instance such that $P$ does not occur
in $I$. So either $P_1$ or $P_2$ does not occur in $I$. So $I$ can be
solved by either $A_1$ or $A_2$. So any CSP instance in \cspneg{P}
can be solved by one of two polynomial algorithms. So $P$ is
tractable.
\end{enumerate}
\end{proof}

The following lemma concerns a pattern in which some structure is
imposed on domain elements. It is essential for our two-constraint
dichotomy.

Let $2V$ be pattern on three variables $v_0$, $v_1$ and $v_2$ with
three points $a,b,c \in A_{v_1}$, three points $d,e,f \in A_{v_2}$
and six points $g,h,i,j,k,l \in A_{v_0}$, such that $a$ is compatible
with $h$, $b$ is compatible with $g$ and $h$, $c$ is incompatible
with $i$, $d$ is incompatible with $j$, $e$ is compatible with $k$
and $l$, $f$ is compatible with $l$. The pattern $2V$ also has the
associated structure ($a\neq b$ or $g\neq h$) and ($e\neq f$ or
$k\neq l$). When a pattern has an associated structure given by a
property $\mathcal{P}$, the property $\mathcal{P}$ must be preserved
by reduction operations. For example, if $\mathcal{P}$ is $a\neq b$
then the points $a$ and $b$ cannot be merged during a reduction. It
is worth pointing out that in a CSP instance, all points are assumed
to be distinct and hence a property such as $a \neq b$ is necessarily
satisfied.

\begin{lemma}\label{twovee}
$2V$ is intractable.
\end{lemma}
\begin{proof}

Let the gadget $V^{+}$ be the pattern on two variables $v_0,v_1$ with
points $a \in A_{v_0}$ and $b,c \in A_{v_1}$ such that $a$ is
compatible with both $b$ and $c$, together with the structure $b\neq
c$. In the pattern $2V$, either $b$ is compatible with two different
points $g$ and $h$, or $h$ is compatible with two different points
$a$ and $b$. So, if $2V$ occurs in a CSP instance on variables
$v'_0,v'_1,v'_2$, then the gadget $V^{+}$ necessarily occurs in the
constraint between $v'_0$ and $v'_1$. By an identical argument, the
gadget $V^{+}$ must also occur in the constraint between $v'_0$ and
$v'_2$.

We define an equality constraint between two variables $v$ and $v'$
with the same domain as the constraint
consisting of compatibility edges between identical values in $A_v$ and
$A_{v'}$ and incompatibility edges between all
other couples of points. Thus, by definition, a point in an equality
constraint is compatible with one and only one point. Since the gadget
$V^{+}$ contains a point $a$ compatible with two different points, $V^{+}$
does not occur in an equality constraint.

We will reduce CSP to \cspneg{2V}. Let $I$ be a CSP instance. For
each $(v,w)$ in $I$ such that there is a non-trivial constraint
between $v$ and $w$, we introduce two new variables $v'$ and $w'$
such that the domain of $v'$ is the same as the domain of $v$, the
domain of $w'$ is the same as the domain of $w$. We add equality
constraints between $v$ and $v'$, and between $w$ and $w'$, and we
add between $v'$ and $w'$ the same constraint as there was between
$v$ and $w$. All other constraints involving $v'$ or $w'$ are
trivial. We also replace the constraint between $v$ and $w$ by a
trivial constraint. After this transformation, $v$ and $w'$ are the
only variables which share a constraint with $v'$. Let $I'$ be the
instance obtained after all such transformations have been performed
on $I$. By construction, $I'$ has a solution if and only $I$ has a
solution.

We now suppose that we have three variables $v_0$, $v_1$ and $v_2$ in
$I'$ such that there are non-trivial constraints between $v_0$ and
$v_1$ and between $v_0$ and $v_2$. By construction, at least one of
these constraints is an equality constraint. Hence, the gadget
$V^{+}$ cannot occur in both of these constraints. It follows that
$2V$ cannot occur in $I'$. So we have reduced $I$ to an instance
without any occurrence of the pattern $2V$. This polynomial reduction
from CSP to \cspneg{2V} shows that $2V$ is intractable.
\end{proof}

\section{Two-Constraint patterns}
\label{sec:2-constraint}

Let $T$ be the following set $\{T_1,T_2,T_3,T_4,T_5\}$:

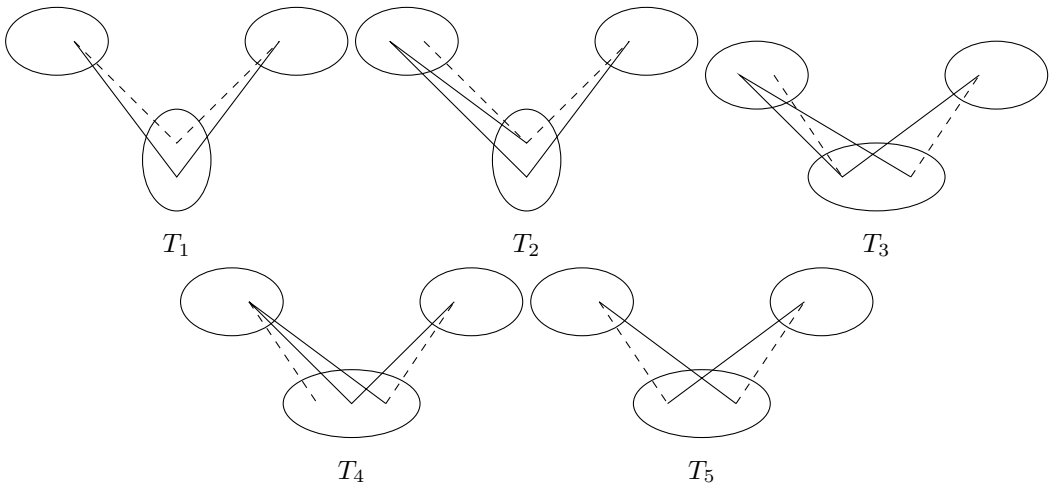
\begin{center}
\begin{tikzpicture}[scale=.45]
\node at (4,-2) {$T_1$};
\draw (0.5,4) circle (1.5 and 1);
\draw (7.5,4) circle (1.5 and 1);
\draw (4,0.5) circle (1 and 1.5);
\draw[style=dashed] (1,4)--(4,1)--(7,4);
\draw (1,4)--(4,0)--(7,4);
\end{tikzpicture}
\begin{tikzpicture}[scale=.45]
\node at (4,-2) {$T_2$};
\draw (0.5,4) circle (1.5 and 1);
\draw (7.5,4) circle (1.5 and 1);
\draw (4,0.5) circle (1 and 1.5);
\draw[style=dashed] (1,4)--(4,1)--(7,4);
\draw (4,1)--(0,4)--(4,0)--(7,4);
\end{tikzpicture}
\begin{tikzpicture}[scale=.45]
\node at (5,-1) {$T_3$};
\draw (1.5,4) circle (1.5 and 1);
\draw (8.5,4) circle (1.5 and 1);
\draw (5,1) circle (2 and 1);
\draw[style=dashed] (2,4)--(4,1)
(6,1)--(8,4);
\draw (8,4)--(4,1)--(1,4)--(6,1);
\end{tikzpicture}
\begin{tikzpicture}[scale=.45]
\node at (5,-1) {$T_4$};
\draw (1.5,4) circle (1.5 and 1);
\draw (8.5,4) circle (1.5 and 1);
\draw (5,1) circle (2 and 1);
\draw[style=dashed] (2,4)--(4,1)
(6,1)--(8,4);
\draw (8,4)--(5,1)--(2,4)--(6,1);
\end{tikzpicture}
\begin{tikzpicture}[scale=.45]
\node at (5,-1) {$T_5$};
\draw (1.5,4) circle (1.5 and 1);
\draw (8.5,4) circle (1.5 and 1);
\draw (5,1) circle (2 and 1);
\draw[style=dashed] (2,4)--(4,1)
(6,1)--(8,4);
\draw (4,1)--(8,4)
(6,1)--(2,4);
\end{tikzpicture}
\end{center}

No pattern in $T$ can be reduced to a different pattern in $T$. As we
will show, each $T_i$ defines a tractable class of binary CSP
instances. For example, $T_4$ defines a class of instances which
includes as a proper subset all instances with zero-one-all
constraints~\cite{zoa}. Zero-one-all constraints can be seen as a
generalisation of 2SAT clauses to multi-valued logics.

Let $2I$ represent the pattern composed of two separate copies of
$1I$, i.e. four points $a,b,c,d$ such that $var(a)$, $var(b)$,
$var(c)$, $var(d)$ are all distinct and both $a,b$ and $c,d$ are
pairs of incompatible points.

\begin{theorem}
Let $P$ be a pattern on two constraints. Then $P$ is tractable if and
only if $P$ is reducible to one of the patterns in $T \cup \{2I\}$.
\end{theorem}

\begin{proof}
$\Rightarrow$: \ \ A two-constraint pattern involves either three or
four distinct variables. Consider first the latter case, in which $P$
is composed of two separate one-constraint patterns $P_1$ and $P_2$
on four distinct variables. By Lemma~\ref{lem:separate_patterns}, $P$
is tractable if and only if both $P_1$ and $P_2$ are tractable.
Furthermore, by Lemma~\ref{lem:1_constraint_dichotomy}, all tractable
one-constraint patterns are reducible to $1I$. Thus, if $P$ is
tractable, then it is reducible to $2I$, by a combination of the two
reductions of $P_1$ and $P_2$ to $1I$. It only remains to study
two-constraint patterns on \emph{three} variables.

From Lemma~\ref{twored}, we know that we only have to study patterns
with at most one incompatibility edge in each constraint. If one of
the constraints does not contain any incompatibility edge at all,
then the pattern is reducible by merging and/or dp-elimination to a
pattern with only one constraint. So we can assume from now on that
there is exactly one incompatibility edge $(a\in A_{v_0},b\in
A_{v_1})$ between $v_0$ and $v_1$, and also exactly one
incompatibility edge $(c\in A_{v_0},d\in A_{v_2})$ between $v_0$ and
$v_2$. The ``skeleton'' of incompatibility edges of an irreducible
tractable pattern can thus take two forms according to whether $a=c$
or $a \neq c$.

From Lemma~\ref{onegreenpoint} we know that $|A_v| \leq 2$ for each
variable $v$ with only one explicitly incompatible point, and that
$|A_v| \leq 3$ for each variable $v$ with two explicitly incompatible
points.
We know from Lemmas~\ref{zed} and~\ref{twovee} that both $Z$ and $2V$
are intractable, so we must look for patterns in which neither one
occurs. We know that we have two possible incompatibility skeletons
to study, each one implying a maximum number of points appearing in
the pattern.

First incompatibility skeleton:
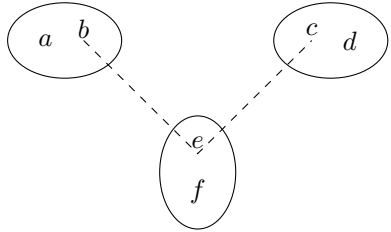
\begin{center}
\begin{tikzpicture}[scale=.5]
\node at (0,4) {$a$};
\node at (1,4.3) {$b$};
\node at (7,4.3) {$c$};
\node at (8,4) {$d$};
\node at (4,1.3) {$e$};
\node at (4,0) {$f$};
\draw (0.5,4) circle (1.5 and 1);
\draw (7.5,4) circle (1.5 and 1);
\draw (4,0.5) circle (1 and 1.5);
\draw[style=dashed] (1,4)--(4,1)--(7,4);
\end{tikzpicture}
\end{center}

Suppose first that $a$ is a point in the pattern. Then there must be
a compatibility edge between $a$ and $e$, otherwise we could merge
$a$ and $b$. There also must be a compatibility edge between $a$ and
$f$, otherwise $a$ would be a dangling point. Similarly, if $d$ is a
point in the pattern, then there must be compatibility edges between
$d$ and $e$, and between $d$ and $f$. So if both $a$ and $d$ are
points in the pattern, then the pattern $2V$ occurs. So $a$ and $d$
cannot be both points of the pattern. Since they play symmetric
roles, we only have two cases to consider: either $a$ is a point in
the pattern and not $d$, or neither $a$ nor $d$ is a point in the
pattern.

If $a$ is a point in the pattern, then the only remaining edges are
$\{f,b\}$ and $\{f,c\}$. $\{f,b\}$ cannot be a compatibility edge,
because otherwise the pattern $Z$ would occur. $\{f,c\}$ must be a
compatibility edge, otherwise we could merge $f$ and $e$. On the
other hand, if neither $a$ nor $d$ is a point in the pattern, then
the only remaining edges are $\{f,b\}$ and $\{f,c\}$. If one of them
is a compatibility edge but not the other, then $f$ would be a
dangling point. So either both $\{f,b\}$ and $\{f,c\}$ are
compatibility edges, or neither of them is. However, the latter case
is reducible to the former one. So the only possible irreducible
tractable patterns are $T_1$ and $T_2$.

Second incompatibility skeleton:

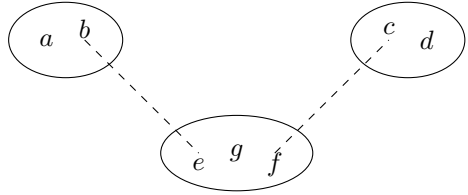
\begin{center}
\begin{tikzpicture}[scale=.5]
\node at (0,4) {$a$};
\node at (1,4.3) {$b$};
\node at (9,4.3) {$c$};
\node at (10,4) {$d$};
\node at (4,0.7) {$e$};
\node at (6,0.7) {$f$};
\node at (5,1) {$g$};
\draw (0.5,4) circle (1.5 and 1);
\draw (9.5,4) circle (1.5 and 1);
\draw (5,1) circle (2 and 1);
\draw[style=dashed] (1,4)--(4,1)
(6,1)--(9,4);
\end{tikzpicture}
\end{center}

If $g$ is a point in the pattern, then there must be a compatibility
edge between $g$ and $b$, otherwise we could merge $g$ and $e$. There
also must be a compatibility edge between $g$ and $c$, otherwise we
could merge $g$ and $f$. We suppose $a$ is a point in the pattern.
Then there is a compatibility edge between $a$ and $e$, otherwise we
could merge $a$ and $b$. There is also a compatibility edge either
between $a$ and $f$ or between $a$ and $g$, otherwise $a$ would be a
dangling point. We cannot have a compatibility edge between $a$ and
$g$, otherwise the pattern $Z$ would occur. So there is a
compatibility edge between $a$ and $f$. There is a compatibility edge
either between $b$ and $f$ or between $c$ and $e$, otherwise we could
merge $e$ and $f$. We cannot have a compatibility edge between $b$
and $f$, otherwise the pattern $Z$ would occur. We cannot have a
compatibility edge between $c$ and $e$, otherwise the pattern $2V$
would occur. So $a$ cannot be a point in the pattern. Since $a$ and
$d$ play symmetric roles, we can also deduce that $d$ cannot be a
point in the pattern. So the only remaining edges are $\{b,f\}$ and
$\{c,e\}$. At least one of them is a compatibility edge, otherwise we
could merge $e$ and $f$. If both of them are compatibility edges, the
pattern $2V$ occurs. So exactly one of them is a compatibility edge.
Since they play symmetric roles, we can assume for instance that
$\{b,f\}$ is a compatibility edge while $\{c,e\}$ is an unknown edge.

If $g$ is not a point in the pattern, then we suppose that $a$ is a
point in the pattern. There is a compatibility edge between $a$ and
$e$, otherwise we could merge $a$ and $b$. There is also a
compatibility edge between $a$ and $f$, otherwise $a$ would be a
dangling point. Similarly, if $d$ is a point in the pattern, then
there must be compatibility edges between $d$ and $e$, and between
$d$ and $f$. So $a$ and $d$ cannot be both points of the pattern.
Since they play symmetric roles, we only have two cases to consider:
either $a$ is a point in the pattern and not $d$, or neither $a$ nor
$d$ is a point in the pattern.

If $a$ is a point in the pattern, then the only remaining edges are
$\{b,f\}$ and $\{c,e\}$. At least one of them is a compatibility
edge, otherwise we could merge $e$ and $f$. There is no compatibility
edge between $b$ and $f$, otherwise the pattern $Z$ would occur. So
there is a compatibility edge between $c$ and $e$.

If neither $a$ nor $d$ is a point in the pattern, then the only
remaining edges are $\{b,f\}$ and $\{c,e\}$. At least one of them is
a compatibility edge, otherwise we could merge $e$ and $f$. So either
exactly one of them is a compatibility edge, or they both are.
However, the former case is reducible to the latter. So the only
possible patterns are $T_3$, $T_4$ or $T_5$.

So if $P$ is a tractable pattern on two constraints, then $P$ is
reducible to one of the patterns in $T$.

$\Leftarrow$: \ \ We now give the tractability proofs for all
patterns in $T$. We assume throughout that we have applied until
convergence the preprocessing operations: arc consistency,
neighbourhood substitution and single-valued variable elimination.

\paragraph{Proof of tractability of $T_1:$}
We suppose we forbid the pattern $T_1$. Let the gadget $X$ be the pattern
on two variables $v_0,v_1$ with points
$a,b \in A_{v_0}$ and $c,d \in A_{v_1}$ such that
$a$ is incompatible with $c$ and compatible with $d$, and $b$ is
compatible with $c$ and incompatible with $d$.

Suppose that the gadget $X$ occurs in an instance. Suppose $a$ is in a solution
$S$. Let $e\in A_{v_2}$ be such that $v_2\neq v_0$, $v_2\neq v_1$ and
$e\in S$. Let $f$ be the point of $S$ in $v_1$.

If $b$ is incompatible with $e$ then $a$, $b$, $d$ and $e$ form the
forbidden pattern. So $b$ is compatible with $e$. Similarly, if $c$
is incompatible with $e$, then $a$, $c$, $f$ and $e$ form the
forbidden pattern. So $c$ is compatible with $e$. So if we replace
$a$ by $b$ and $f$ by $c$ in $S$, then we have another solution. So
if $a$ is in a solution, then $b$ is also in a solution. So we can
remove $a$ while preserving the solvability of the instance.

So we can assume from now on that the gadget $X$ doesn't occur in
the instance. The following lemma indicates when we can perform fusion
operations.

\begin{lemma} \label{lem:T1notintroduced}
Consider a (simple or complex) fusion of two variables $v,v'$ in an
instance in \cspneg{T_1}. Suppose that whenever $(a,a')$ and $(b,b')$
are pairs of fused points during this fusion, such that $a \neq b \in
A_v$ and $a' \neq b' \in A_{v'}$, then either $a$ and $b'$ are
incompatible or $b$ and $a'$ are incompatible. Then the pattern $T_1$
cannot be introduced by this fusion.
\end{lemma}

\begin{proof}
By the definition of (simple or complex) fusion, the only way that
$T_1$ could be introduced is when the two points in the central
variable of $T_1$ are created by the fusion of pairs of points
$(a,a')$ and $(b,b')$ such that the compatibility of the points $a,b
\in A_v$ and $a',b' \in A_{v'}$ with the two other points $a_1$,
$a_2$ of $T_1$ are as shown:
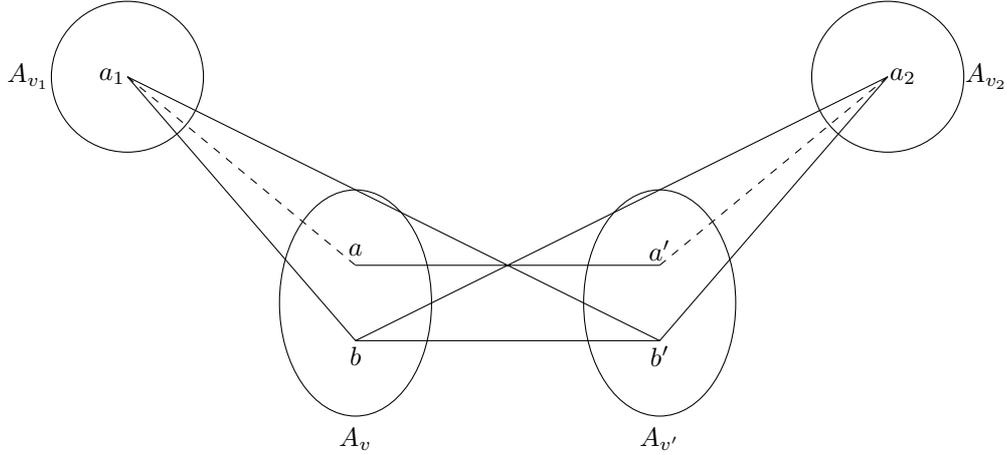
\begin{center}
\begin{tikzpicture}
\node at (-1.3,6) {$A_{v_1}$}; \node at (11.3,6) {$A_{v_2}$}; \node at
(3,1.2) {$A_v$}; \node at (7,1.2) {$A_{v'}$}; \node at (-0.2,6) {$a_1$};
\node at (10.2,6) {$a_2$}; \node at (3,3.7) {$a$}; \node at (7,3.7)
{$a'$}; \node at (3,2.3) {$b$}; \node at (7,2.3) {$b'$};
\draw[style=dashed] (0,6)--(3,3.5) (7,3.5)--(10,6); \draw[]
(0,6)--(3,2.5) (0,6)--(7,2.5) (3,3.5)--(7,3.5) (3,2.5)--(7,2.5)
(7,2.5)--(10,6) (3,2.5)--(10,6); \draw (3,3) circle (1 and 1.5);
\draw (7,3) circle (1 and 1.5); \draw (0,6) circle (1); \draw (10,6)
circle (1);
\end{tikzpicture}
\end{center}
Now, if $a$ and $b'$ were incompatible, then $T_1$ was already
present on points $a_1$, $a$, $b$, $b'$ in the original instance, and
hence cannot be introduced by the fusion. Similarly, if $b$ and $a'$
were incompatible, then $T_1$ was already present on points $b$,
$a'$, $b'$, $a_2$ in the original instance.
\end{proof}

\begin{definition}
$\forall v,v'$, $\forall a,b\in A_v$, we say that $b$ is \emph{better
than} $a$ \emph{with respect to} $v'$, which we denote by $a\leq b$
for $(v,v')$ (or for $v'$), if every point in $A_{v'}$ compatible
with $a$ is also compatible with $b$.
\end{definition}

It is easy to see that $\leq$ is a partial order.

\begin{remark}
We also have the relations $\geq$, <,> and =, derived in the obvious
way from $\leq$.
\end{remark}

\begin{lemma}\label{triord}
\begin{enumerate}
\item $\forall (v,v')$, the order $\leq$ on $A_v$ with respect to $v'$ is total.
\item $\forall v$, $\forall a,b\in A_v$, there is $v'$ such that $a<b$ for $v'$.
\item $\forall v$, $\forall a,b\in A_v$, there is only one $v'$ such that $a<b$ for $v'$.
\end{enumerate}
\end{lemma}

\begin{proof}
\begin{enumerate}
\item Because the gadget $X$ cannot occur.
\item Otherwise $b$ is dominated by $a$ and we can remove it by neighbourhood
substitution.
\item Because of the initial forbidden pattern.
\end{enumerate}
\end{proof}

\begin{lemma}\label{revord}
If $a<b<c$ for $(v_0,v_1)$, then there exists $v_2\neq v_1$ such that
$c<b<a$ for $(v_0,v_2)$.
\end{lemma}

\begin{proof}
Since we have $a<b$ for $(v_0,v_1)$, from Lemma~\ref{triord}.2 there
is some $v_2$ such that $b<a$ for $(v_0,v_2)$. Since $b<c$ for
$(v_0,v_1)$, $c\leq b$ for $(v_0,v_2)$ by Lemma~\ref{triord}.3. If
$c<b$ for $v_2$, then we have the lemma. Otherwise, we have $c=b<a$
for $v_2$. Since $b<c$ for $v_1$, there exists $v_3\neq v_1,v_2$ such
that $c<b$ for $v_3$. Since $a<b$ for $v_1$, $b\leq a$ for $v_3$. So
$c<b\leq a$ for $v_3$. So we have $c<a$ for both $v_2$ and $v_3$,
which is not possible. So we must have $c<b<a$ in $v_2$.
\end{proof}

\begin{lemma}\label{striord}
$\forall a,b,c,d\in A_{v_0}$, for all $v_1 \neq v_0$ none of the
following is true:
\begin{enumerate}
\item $a=b<c<d$ \ for $v_1$.
\item $a<b=c<d$ \ for $v_1$.
\item $a<b<c=d$ \ for $v_1$.
\end{enumerate}
\end{lemma}

\begin{proof}
We give the proof only for the case 1, since the proofs of cases
2 and 3 are almost identical. 
Since we have $a<c<d$ for $v_1$, from Lemma~\ref{revord} there exists
$v_2$ such that $d<c<a$ for $v_2$. Likewise, since $b<c<d$ for $v_1$,
there exists $v_2'$ such that $d<c<b$ for $v_2'$. Since $d<c$ for
both $v_2$ and $v_2'$, $v_2=v_2'$ by Lemma~\ref{triord}.3. This leaves three possibilities:
\begin{enumerate}
\item $d<c<b<a$ for $v_2$: from Lemma~\ref{revord} we know there is $v_3$
such that $a<b<c$ for $v_3$. So we have $a<c$ for both $v_1$ and
$v_3$ with $v_1\neq v_3$ (since $a=b$ for $v_1$), which is not
possible by Lemma~\ref{triord}.3. So we cannot have this
possibility.
\item $d<c<b=a$ for $v_2$: since $a=b$ for both $v_1$ and $v_2$,
by Lemma~\ref{triord}.2 there
is a different $v_3$ such that $a<b$ for $v_3$. Since $c<b$ for $v_2$
and $v_3\neq v_2$, $b\leq c$ for $v_3$. So $a<c$ for $v_3$. But we
also have $a<c$ for $v_1$ and $v_1\neq v_3$. So by
Lemma~\ref{triord}.3 we cannot have this possibility.
\item $d<c<a<b$ for $v_2$: equivalent to the case $d<c<b<a$ after interchanging
$a$ and $b$.
\end{enumerate}
\end{proof}

\begin{corollary}\label{revstri}
If for some $(v_0,v_1)$, we have at least three equivalence classes
in the order on $A_{v_0}$ with respect to $v_1$ then:
\begin{enumerate}
\item The order on $A_{v_0}$ with respect to $v_1$ is strict.
\item There is $v_2$ such that the order on $A_{v_0}$ with respect to $v_2$
is the exact opposite to the order on $A_{v_0}$ with respect to $v_1$.
\item $\forall v_3$ such that $v_3\neq v_0,v_1,v_2$, there is only one equivalence
class in the order on $A_{v_0}$ with respect to $v_3$.
\end{enumerate}
\end{corollary}

\begin{proof}
Points 1,2 and 3 follow respectively from Lemma~\ref{striord},
Lemma~\ref{revord} and Lemma~\ref{triord}.
\end{proof}

\begin{lemma}\label{twotwo}
$\forall a,b,c,d\in A_{v_0}$, there is no $v_1$ such that $a=b<c=d$
for $v_1$.
\end{lemma}

\begin{proof}
By Lemma~\ref{triord}.2, we know there is some $v_2$
such that $a<b$ for $v_2$. Since we have $a<c$ and $a<d$ for $v_1$,
by Lemma~\ref{triord}.3, we have $c\leq a$ and $d\leq a$ for $v_2$. From
Corollary~\ref{revstri}, we cannot have $c<a<b$ or $d<a<b$ for $v_2$,
so we have $d=c=a<b$ for $v_2$. Since we have $c=d$ for both $v_1$
and $v_2$, we have a different variable $v_3$ such that $c<d$ for
$v_3$. Since $c<b$ for $v_2$ and $v_3\neq v_2$, $b\leq c$ for $v_3$.
So $b<d$ for $v_3$. But we also have $b<d$ for $v_1$ and $v_1\neq
v_3$. So, by Lemma~\ref{triord}.3, we cannot have this possibility.
\end{proof}

\begin{lemma} \label{troistiers}
If for some $(v,v')$ there are at least three equivalence classes in
the order on $A_v$ with respect to $v'$, then there are the same
number of points in both $A_v$ and $A_{v'}$ and both the order on
$A_v$ with respect to $v'$ and the order on $A_{v'}$ with respect to
$v$ are strict.
\end{lemma}

\begin{proof}
Let $d$ be the number of points in $A_v$ and $d'$ the number of
points in $A_{v'}$. From Lemma~\ref{striord} we know that the order on
$A_v$ with respect to $v'$ is strict. So we have $a_1<a_2<\dots<a_d$
for $(v,v')$. So we have $(a'_1,a'_2,\dots,a'_{d-1})$ such that
$\forall 1\leq i<d$, $a_i$ and $a'_i$ are incompatible but $a_{i+1}$
and $a'_i$ are compatible. So $\forall 2\leq i<d$ we have $a_i$ and
$a'_i$ which are incompatible but $a_i$ and $a'_{i-1}$ are
compatible. So, by Lemma~\ref{triord}.1 we have
$a'_1>a'_2>\dots>a'_{d-1}$ for $v$. Moreover, since
$a_1$ is incompatible with $a'_1$, $a_1$ is incompatible with all
$a'_i$ for $1\leq i<d$. By arc consistency, we have $a'_0$ such that
$a_1$ and $a'_0$ are compatible. So we have
$a'_0>a'_1>a'_2>\dots>a'_{d-1}$. So we have $d\leq d'$ and at least
three equivalence classes in the order on $A_{v'}$ with respect to
$v$. By switching $v$ and $v'$ in the proof, we can prove the
remaining claims of the Lemma.
\end{proof}

We say that the pair of variables $(v,v')$ is a \emph{3-tiers pair}
if there are at least 3 classes of equivalence in the order on $A_v$
with respect to $v'$; we say that it is a \emph{2-tiers pair}
otherwise.

We suppose we have $v$ and $v'$ such that $(v,v')$ is a 3-tiers pair.
Let $d$ be the number of points in $A_v$. From Lemma~\ref{troistiers}
we know that the points in $A_v$ can be denoted $a_1<a_2<\dots<a_d$ for
$v'$ and the points in $A_{v'}$ can be denoted $b_1<b_2<\dots<b_d$
for $v$. We will show that we can perform a simple fusion of $v$ and
$v'$ with fusion function $f$ given by $f(a_i)=b_{d+1-i}$
($i=1,\ldots,d$).

\begin{lemma}\label{compa}
$\forall 1\leq i\leq d$, $\{b_{d+1-i},b_{d+1-i+1},\dots,b_d\}$ is the
exact set of points compatible with $a_i$.
\end{lemma}

\begin{proof}
If we have $a_i<a_j$ for $v'$, it means $a_i$ is compatible with
strictly less points in $A_{v'}$ than $a_j$. By arc consistency,
every point in $A_v$ is compatible with a point in $A_{v'}$. So
$\forall 1\leq i\leq d$, we have $d$ possibilities $(1,2,\dots,d)$
for the number of points compatible with $a_i$. Since we have $d$
points in $A_v$, it means that $\forall 1\leq i\leq d$, $a_i$ is
compatible with $i$ points in $A_{v'}$. By definition of the order on
a variable with respect to another variable, the points in $A_{v'}$
compatible with a point $a_i\in A_v$ are the greatest points for $v$.
So we have the Lemma.
\end{proof}

\begin{lemma}\label{solv}
$\forall 1\leq i\leq d$, if $a_i$ is in a solution $S$, then there is
a solution $S'$ such that both $b_{d+1-i}$ and $a_i$ are in $S'$.
\end{lemma}

\begin{proof}
Let $b$ be the point of $S$ in $v'$. If $b_{d+1-i}=b$, then we have
the result. Otherwise, let $c\neq b$ be a point of $S$. If $c=a_i$,
then from Lemma~\ref{compa} we know that $c$ is compatible with
$b_{d+1-i}$. Otherwise, let $v_c = var(c)$. So $v_c\neq v$. From
Lemma~\ref{compa} we have $b_{d+1-i}<b$ for $v$. So $b\leq b_{d+1-i}$
for $v_c$. So $b_{d+1-i}$ is compatible with $c$. So $b_{d+1-i}$ is
compatible with all the points in $S$. So we have a solution $S'$
obtained by replacing $b$ by $b_{d+1-i}$ in $S$ which contains both
$a_i$ and $b_{d+1-i}$.
\end{proof}

We now perform the simple fusion of $v$ and $v'$ by with fusion
function $f(a_i)=b_{d+1-i}$ for $1\leq i\leq d$. By Lemma~\ref{solv},
this is a valid simple fusion and by Lemma~\ref{lem:fusion} and the following
lemma the resulting instance is in \cspneg{T_1} and
solvable if and only if the original instance was solvable.

\begin{lemma}
The simple fusion of $v$ and $v'$ in an instance in \cspneg{T_1},
where $(v,v')$ is a 3-tiers pair, does not create the forbidden
pattern.
\end{lemma}

\begin{proof}
Let $a,b$ be two distinct points in $A_v$. Without loss of
generality, suppose that $a < b$ for $v'$. By choice of the fusion
function $f$, $b$ is the smallest (according to the order $<$ for $v'$)
of the points in $A_v$ compatible with $f(b)$.
Therefore, $a$ and $f(b)$ are incompatible. The result then follows
from Lemma~\ref{lem:T1notintroduced}.
\end{proof}

From now on, $\forall (v,v')$, we can assume that each pair $(v,v')$
is a 2-tiers pair. We call \emph{winner} for $(v,v')$ the points in
the greater equivalence class in the order for $(v,v')$. The other
points are called \emph{losers} for this order. A same point can (and
actually will) be a winner for a given order and a loser for another
order. If for a given order there is only one equivalence class, then
all the points are considered winners.

The winners for $(v,v')$ are compatible with all the points in
$A_{v'}$. The losers for $(v,v')$ are only compatible with the
winners for $(v',v)$.

We say that a variable $v$ is \emph{one-winner} if $\forall v'\neq
v$, either only one point of $A_v$ is a winner for $(v,v')$ or all
the points in $A_v$ are. Similarly, we say that a variable $v$ is
\emph{one-loser} if $\forall v'\neq v$, either only one point of
$A_v$ is a loser for $(v,v')$ or all the points of $A_v$ are winners
for $(v,v')$.

\begin{lemma}\label{previous}
$\forall v$, if there is $v'$ such that there is only one winner for
$(v,v')$, then $v$ is one-winner. Similarly, if there is $v'$ such
that there is only one loser for $(v,v')$, then $v$ is one-loser.
\end{lemma}

\begin{proof}
Let $a,b,c,d,e,f\in A_v$ be such that there are $v_1\neq v_2$ with
$a=b<c$ for $v_1$, $d<e=f$ for $v_2$, $a\neq b$ and $e\neq f$. If
$d\neq c$, then from Lemma~\ref{twotwo}, we have $a=b=d<c$ for $v_1$
and $d<e=f=c$ for $v_2$. So $d<c$ for both $v_1$ and $v_2$ with
$v_1\neq v_2$ (which is a contradiction by Lemma~\ref{triord}.3). So
we cannot have $d\neq c$. So $d=c$. So we have $c<e=f$ for $v_2$.
From Lemma~\ref{twotwo} we have $c<e=f=a=b$ for $v_2$. Since we have
$a=b$ for both $v_1$ and $v_2$, by Lemma~\ref{triord}.2 there is a
different variable $v_3$ such that $a<b$ for $v_3$. Since $a<c$ for
$v_1$, $c\leq a$ for $v_3$. So $c<b$ for $v_3$. So $c<b$ for both
$v_2$ and $v_3$ with $v_2\neq v_3$. This is impossible by Lemma~\ref{triord}.3.
So we have the Lemma.
\end{proof}

\begin{corollary}\label{follow}
$\forall v$, either $v$ is one-winner or $v$ is one-loser.
\end{corollary}

\begin{proof}
Lemma~\ref{triord}.2 tells us that there exists $v'$ and $a,b\in
A_v$ such that $a<b$ for $v'$. By Lemma~\ref{twotwo}, either there is
only one winner for $(v,v')$ or only one loser. The result follows
directly from Lemma~\ref{previous}.
\end{proof}

Let $E$ be the set of one-winner variables and $F=V\backslash E$ with
$V$ being the set of all variables. From Corollary~\ref{follow}, the
variables in $F$ are one-loser. Let $v_a,v_b\in E$ be such that there
is a non-trivial constraint between $v_a$ and $v_b$. Since $v_a\in
E$, there is only one winner $a$ for $v_b$ in $v_a$. Similarly, there
is only one winner $b$ for $v_a$ in $v_b$. We can perform a complex
fusion of $v_a$ and $v_b$ with hinge value $a$ and fusion function
the constant function $f=b$.

By Lemma~\ref{lem:fusion}, the instance resulting from this fusion is
solvable if and only if the original instance was solvable.

\begin{lemma}
The complex fusion of two one-winner variables $v_a$ and $v_b$ in an
instance of \cspneg{T_1} does not create the forbidden pattern.
\end{lemma}

\begin{proof}
Suppose that $(c,c')$ and $(d,d')$ are corresponding pairs of points
during this fusion, with $c \neq d \in A_{v_a}$ and $c' \neq d' \in
A_{v_b}$. Since $v_a$ only has one winner for $v_b$, we know that
either $c$ or $d$ is a loser for $v_b$. Without loss of generality,
suppose $d$ is a loser for $v_b$. Since $v_b$ only has one winner for
$v_a$, and losers are only compatible with winners, we know that $d$
is incompatible with $c'$ (since it is necessarily compatible with
$d'$ for the fusion to take place). The result now follows directly
from Lemma~\ref{lem:T1notintroduced}.
\end{proof}

We have shown that we can fusion any pair of variables in $E$
between which there is a non-trivial constraint. We now do the same for $F$.

Let $E$ be the set of one-winner variables and $F=V\backslash E$ with
$V$ being the set of all variables. From Corollary~\ref{follow}, we
know all variables in $F$ are one-loser. Let $v_a,v_b\in F$ be such
that there is a non-trivial constraint between $v_a$ and $v_b$.
Since there is a non-trivial
constraint between $v_a$ and $v_b$, there is some $a\in A_{v_a}$ and
some $b\in A_{v_b}$ such that $a$ is incompatible with $b$.

\begin{lemma}\label{noinspi}
If $a'\in A_{v_a}$ is in a solution $S$ and $a'\neq a$, then $b$ is
in a solution $S'$ containing $a'$.
\end{lemma}

\begin{proof}
Let $b'$ be the point of $S$ in $v_b$. If $b'=b$, then we have the
result. Since $v_a$ is a one-loser variable, we know that all points
in $A_{v_a}$ other than $a$ are winners. Thus $a'$ is compatible with
$b$. By a symmetric argument, $b'$ is compatible with $a$. If we have
$c\in S$ such that $b$ is incompatible with $c$,
then $a$, $b'$, $c$ and $b$ form the forbidden pattern. So $b$ is
compatible with all the points in $S$. So if we replace $b'$ by $b$
in $S$ we get a solution $S'$ containing both $a'$ and $b$.
\end{proof}

It follows from Lemma~\ref{noinspi} that we only need to consider
solutions containing $a$ or $b$. We can therefore perform a complex
fusion of $v_a$ and $v_b$ with hinge value $a$ and fusion function
the constant function $f=b$.

\begin{lemma}
The complex fusion of $v_a$ and $v_b$ in an instance of \cspneg{T_1}
does not create the forbidden pattern.
\end{lemma}

\begin{proof}
In all pairs $(c,c')$ of corresponding points in this fusion, we must
have either $c=a$ or $c'=b$. Suppose that $(c,c')$ and $(d,d')$ are
corresponding pairs of points during the fusion, with $c \neq d \in
A_{v_a}$ and $c' \neq d' \in A_{v_b}$. Without loss of generality, we
can assume that $c=a$ and $d'=b$. But we know that $a$ was
incompatible with $b$. The result now follows directly from
Lemma~\ref{lem:T1notintroduced}.
\end{proof}

We say a point $a$ is \emph{weakly incompatible} with a variable $v$
if there exists some $b\in A_v$ such that $a$ is incompatible with
$b$.

The total number of assignments decreases when we fuse variables, so
the total number of (simple or complex) fusions that can be performed
is linear in the size of the original instance. After all possible
fusions of pairs of variables, we have two sets of variables $E$ and
$F=V \setminus E$ such that:
\begin{itemize}
\item $\forall v,v'\in E$, there is no non-trivial constraint between $v$ and $v'$.
\item $\forall v,v'\in F$, there is no non-trivial constraint between $v$ and $v'$.
\item $\forall f$ a point in $A_v$ for some $v \in F$,
$f$ is weakly incompatible with one and only one variable $v' \in E$.
Furthermore, $f$ is incompatible with all points of $A_{v'}$ but one
(since $v' \in E$ is a one-winner variable).
\item The only possible non-trivial constraint between a variable $v_1\in E$ and
another variable
$v_2\in F$ is the following with $d_1$ being the size of the domain of $v_1$:
\begin{itemize}
\item There is a point $b\in A_{v_2}$ incompatible with exactly $d_1-1$ points in
$A_{v_1}$.
\item $\forall b'\in A_{v_2}$ with $b'\neq b$, $b'$ is compatible with all the points
in $A_{v_1}$.
\end{itemize}
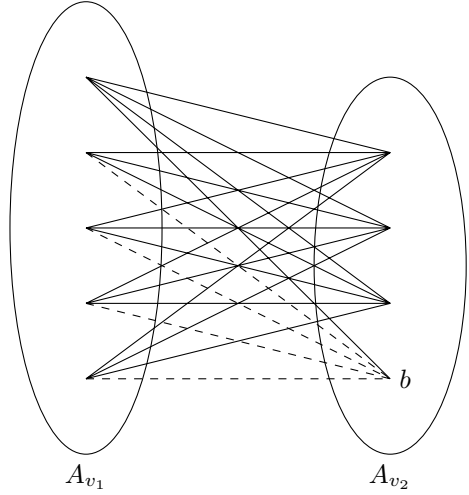
\begin{center}
\begin{tikzpicture}
\node at (3,0.7) {$A_{v_1}$};
\node at (7,0.7) {$A_{v_2}$};
\node at (7.2,2) {$b$};
\draw [style=dashed] (3,3)--(7,2)
(3,4)--(7,2)
(3,2)--(7,2)
(3,5)--(7,2);
\draw[] (3,6)--(7,2)
(7,3)--(3,2)
(7,3)--(3,3)
(7,3)--(3,4)
(7,3)--(3,5)
(7,3)--(3,6)
(7,4)--(3,2)
(7,4)--(3,3)
(7,4)--(3,4)
(7,4)--(3,5)
(7,4)--(3,6)
(7,5)--(3,2)
(7,5)--(3,3)
(7,5)--(3,4)
(7,5)--(3,5)
(7,5)--(3,6);
\draw (3,4) circle (1 and 3);
\draw (7,3.5) circle (1 and 2.5);
\end{tikzpicture}
\end{center}
\end{itemize}

We call NOOSAT (for Non-binary Only Once Sat) the following problem:
\begin{itemize}
\item A set of variables $V=\{v_1,v_2,\dots,v_e\}$.
\item A set of values $A=\{a_1,a_2,\dots,a_n\}$.
\item A set of clauses $C=\{C_1,C_2,\dots,C_f\}$ such that:
\begin{itemize}
\item Each clause is a disjunction of literals, with a literal being in this case of the
form $v_i=a_j$.
\item $\forall i,j,p,q ((v_i=a_j)\in C_p)\wedge ((v_i=a_j)\in C_q)\Rightarrow p=q$.
\end{itemize}
\end{itemize}

\begin{lemma}
CSP($\overline{T_1}$) can be reduced to NOOSAT in polynomial time.
\end{lemma}

\begin{proof}
We suppose we have a binary CSP instance in \cspneg{T_1} and
preprocessed as described above. We have shown that the non-trivial
constraints between variables $v \in F$ and $v'\in E$ are all of the
form $v=b\Rightarrow v'=a$. Furthermore, each variable-value
assignment $v=b$ occurs in exactly one such constraint. For any $v
\in F$, we can replace the set of such constraints $v=b_i\Rightarrow
v_i=a_i$, for all values $b_i$ in the domain of $v$, by the clause
$(v_1 = a_1) \vee \ldots \vee (v_d = a_d)$. It only remains to prove
that no literal appears in two distinct clauses. Suppose that we have
a literal $v_1= a$ which occurs in two distinct clauses. Then there
must have been two constraints $v_2=b \Rightarrow v_1=a$ and $v_3=c
\Rightarrow v_1=a$ and with $v_1 \in E, v_2 \neq v_3 \in F$. Let $a'
\neq a$ be a point in $A_{v_1}$. Then $b$ and $c$ are both
incompatible with $a'$ but compatible with $a$.
But this is precisely the forbidden pattern. This contradiction shows
that CSP($\overline{T_1}$) can be reduced to NOOSAT.
\end{proof}

The constraints in NOOSAT are convex when viewed as
$\{0,\infty\}$-valued cost functions, and the clauses are non
overlapping. So, from~\cite{nonoverlap}, it is solvable in polynomial
time. So the forbidden pattern $T_1$ is tractable.

\paragraph{Proof of tractability of $T_2:$}
Let the gadget $N$ be the following pattern: two variables $v_0,v_1$
with points $a,b \in A_{v_0}$ and $c,d \in A_{v_1}$, such that $a$
and $b$ compatible with $d$, $b$ incompatible with $c$ and with the
structure $a\neq b$.

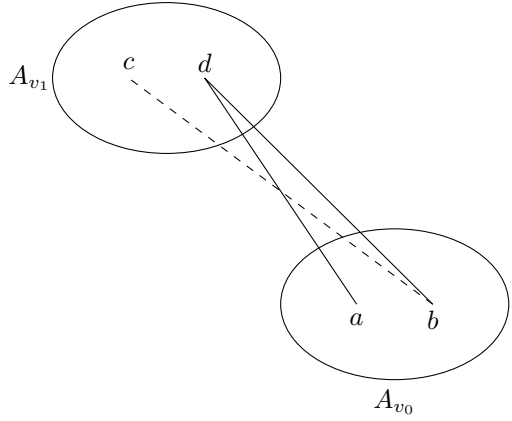
\begin{center}
\scalebox{1}{\begin{tikzpicture}
\node at (0,-0.2) {$b$};
\node at (-3,3.2) {$d$};
\node(c) at (3,3) {};
\node at (-4,3.2) {$c$};
\node at (-1,-0.2) {$a$};
\node at (-0.5,-1.3) {$A_{v_0}$};
\node at (-5.3,3) {$A_{v_1}$};
\draw[style=dashed] (0,0)--(-4,3);
\draw[] (0,0)--(-3,3)
(-3,3)--(-1,0);
\draw (-0.5,0) circle (1.5 and 1);
\draw (-3.5,3) circle (1.5 and 1);
\end{tikzpicture}}
\end{center}

Suppose we have the gadget $N$. Let $v_2$ be a variable with $v_2\neq
v_0$, $v_2\neq v_1$ and let $e$ be a point in $A_{v_2}$ such that $a$
and $e$ are compatible.
If $b$ is incompatible with $e$, then we have the forbidden pattern
$T_2$ on $d$, $c$, $b$, $a$, $e$. So $b$ is compatible with $e$. If
all the points in $A_{v_1}$ which are compatible with $a$ are also
compatible with $b$, then we can remove $a$ by neighbourhood
substitution. So, assuming that neighbourhood substitution operations
have been applied until convergence, if we have the gadget $N$, then
there is a point $g\in A_{v_1}$ compatible with $a$ and incompatible
with $b$.

Let $v_3\neq v_1$. By arc consistency, there is $h\in A_{v_3}$ such
that $h$ is compatible with $a$. If $b$ and $h$ are incompatible,
then we have the forbidden pattern $T_2$ on $d$, $g$, $b$, $a$, $h$.
So $b$ and $h$ are compatible. If there is $i\in A_{v_3}$ such that
$b$ and $i$ are incompatible, then we have the forbidden pattern on
$h$, $i$, $b$, $a$, $g$.
So $b$ is compatible with all the points in $A_{v_3}$. So, if we have
the gadget $N$, then $b$ is compatible with all the points of the
instance outside $v_0,v_1$.

\begin{definition}
A constraint $C$ between two variables $v$ and $v'$ is
\emph{functional} from $v$ to $v'$ if $\forall a\in A_v$, there
is one and only one point in $A_{v'}$ compatible with $a$.
\end{definition}

Let the gadget $V^{-}$ be the pattern comprising three variables
$v_4$, $v_5$, $v_6$ and points $a \in A_{v_4}$, $b \in A_{v_5}$, $c
\in A_{v_6}$ such that $a$ incompatible with both $b$ and $c$.

From now on, since $V^{-}$ is a tractable pattern~\cite{cz11:aij}, we
only need to consider the connected components of the constraint
graph which contain $V^{-}$.

\begin{lemma}\label{1412one}
If in an instance from \cspneg{T_2}, we have the gadget $V^{-}$, then
the constraint between $v_5$ and $v_4$ is functional from $v_5$ to
$v_4$ and the constraint between $v_4$ and $v_6$ is functional from
$v_6$ to $v_4$.
\end{lemma}

\begin{proof}
By symmetry, it suffices to prove functionality from $v_5$ to $v_4$.
We suppose we have the gadget $V^{-}$. Let $d\in A_{v_5}$ be
compatible with $a$.
Since $a$ is weakly incompatible with two different variables, $a$,
$b$ and $d$ cannot be part of the gadget $N$. So the only point in
$A_{v_4}$ compatible with $d$ is $a$. So if a point in $A_{v_4}$ is
compatible with $a$, then it is only compatible with $a$. Likewise,
if a point in $A_{v_6}$ is compatible with $a$, then it is only
compatible with $a$.

Let $f\neq a$ be a point in $A_{v_4}$. By arc consistency, we have
$d\in A_{v_5}$ and $e\in A_{v_6}$ such that $a$ is compatible with
$d$ and with $e$. From the previous paragraph, we know that both $d$
and $e$ are incompatible with $f$.

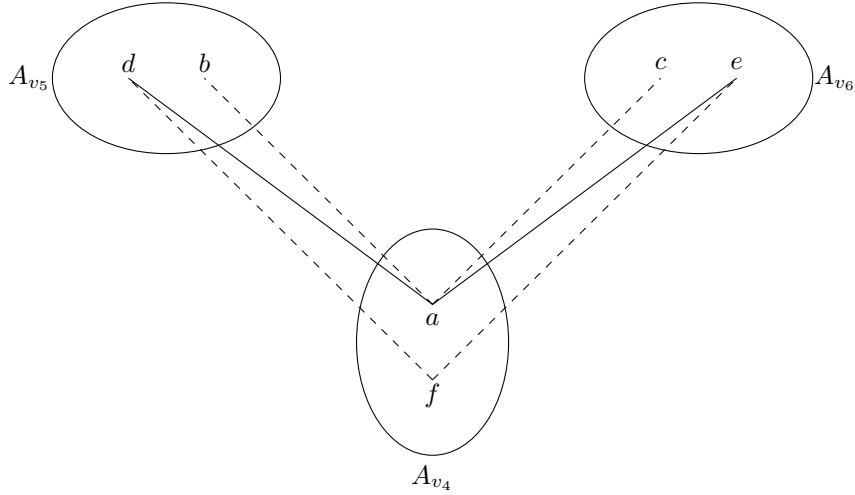
\begin{center}
\scalebox{1}{\begin{tikzpicture}
\node at (0,-0.2) {$a$};
\node at (-3,3.2) {$b$};
\node at (-4,3.2) {$d$};
\node at (0,-1.2) {$f$};
\node at (3,3.2) {$c$};
\node at (4,3.2) {$e$};
\node at (0,-2.3) {$A_{v_4}$};
\node at (-5.3,3) {$A_{v_5}$};
\node at (5.3,3) {$A_{v_6}$};
\draw[style=dashed] (0,0)--(-3,3)
(0,-1)--(-4,3)
(0,-1)--(4,3)
(0,0)--(3,3);
\draw[] (0,0)--(-4,3)
(0,0)--(4,3);
\draw (0,-0.5) circle (1 and 1.5);
\draw (-3.5,3) circle (1.5 and 1);
\draw (3.5,3) circle (1.5 and 1);
\end{tikzpicture}}
\end{center}

So $d$, $e$ and $f$ form the gadget $V^{-}$. So each point in
$A_{v_5}$ and $A_{v_6}$ compatible with $f$ is compatible with only
one point of $A_{v_4}$. So each point in $A_{v_5}$ and $A_{v_6}$
compatible with a point in $A_{v_4}$ is compatible with only one
point of $A_{v_4}$. By arc consistency, each point of $A_{v_5}$ and
$A_{v_6}$ is compatible with exactly one point of $A_{v_4}$. So the
constraint between $v_4$ and $v_5$ is functional from $v_5$ to $v_4$.
\end{proof}

\begin{lemma}\label{1412two}
In a connected component of the constraint graph containing $V^{-}$
of an instance from \cspneg{T_2}, all constraints are either
functional or trivial.
\end{lemma}

\begin{proof}
Let $P(V)$ be the following property: "$V$ is a connected subgraph of
size at least two of the constraint graph and all constraints in $V$
are either functional or trivial".

$P(\{v_4,v_5\}$ is true from Lemma~\ref{1412one}.

Let $V_\text{all}$ be the set of all variables of the connected
subgraph of the constraint graph containing $V^{-}$. Let $V$ be a
maximum (with respect to inclusion) subset of $V_\text{all}$ for
which $P(V)$. Let $V'=V_\text{all}\backslash V$. Let $v'\in V'$. Let
$v\in V$ be such that $C(v,v')$ (the constraint on $v,v'$) is
non-trivial. So there is $d\in A_v$ and $e\in A_{v'}$ such that $d$
and $e$ are incompatible. Since $V$ is connected and of cardinality
at least two, then there is $v'' \in V$ such that $C(v,v'')$ is
functional. By arc consistency and elimination of single-valued
variables, there is necessarily a point $f\in A_{v''}$ such that $d$
and $f$ are incompatible. So $d$, $e$ and $f$ form the gadget
$V^{-}$. From Lemma~\ref{1412one} we know $C(v,v')$ is functional. So
$P(V)$ is true for all subsets of $V_\text{all}$.
\end{proof}

\begin{lemma}\label{1412three}
$\forall v$, all points in $A_v$ are weakly incompatible with the
exact same set of variables.
\end{lemma}

\begin{proof}
Let $a\in A_v$ be weakly incompatible with $v'$. So $C(v,v')$ is non
trivial. So $C(v,v')$ is functional.

If $C(v,v')$ is functional from $v$ to $v'$, then a point in $A_v$
can be compatible with only one point in $A_{v'}$. We can assume, by
elimination of single-valued variables, that there are at least two
points in $A_{v'}$, so every point in $A_v$ is weakly incompatible
with $v'$.

If $C(v,v')$ is functional from $v'$ to $v$, then let $b\neq a$ in
$v$. By arc consistency, we know there is $c\in A_{v'}$ such that $a$
and $c$ are compatible. Since $C(v,v')$ is functional from $v'$ to
$v$, then $c$ is compatible with only one point in $A_v$, in that
case $a$, so $b$ is incompatible with $c$. So every point in $A_v$ is
weakly incompatible with $v'$.

So $\forall (v,v'), a\in A_v$ weakly incompatible with $v'\Rightarrow
\forall b\in A_v, b$ weakly incompatible with $v'$.
\end{proof}

\begin{definition}
A sequence of variables $(v_0,v_1,\dots,v_k)$ is a \emph{path of
functionality} if $\forall 0\leq i\leq k-1: C(v_i,v_{i+1})$ is
functional from $v_i$ to $v_{i+1}$.
\end{definition}

\begin{lemma}
$\forall v,v'$, either $v'$ is a leaf in the constraint graph, or
there is a path of functionality from $v$ to $v'$.
\end{lemma}

\begin{proof}
Since we are in a connected component, there is a path of
incompatibility $(v_0=v,v_1,v_2,\dots,v_k=v')$ with all $v_i$
different. If $v'$ is not a leaf, then we have a path of
incompatibility $(v_0,v_1,v_2,\dots,v_{k-1},v_k,v_{k+1})$ with
$v_{k+1}\neq v_{k-1}$. From Lemma~\ref{1412three} we have a path of
incompatibility $(a_0\in A_{v_0},a_1\in A_{v_1},\dots,a_k\in
A_{v_k},a_{k+1}\in A_{v_{k+1}})$. So $\forall 1\leq i\leq k$,
$a_{i-1}$, $a_i$ and $a_{i+1}$ form the gadget $V^{-}$. So from
Lemma~\ref{1412one}, $\forall 1\leq i\leq k$, $C(v_{i-1},v_{i})$ is
functional from $v_{i-1}$ to $v_i$. So we have a path of
functionality from $v$ to $v'$.
\end{proof}

Leaves can be added to an existing solution by arc consistency. So
once we have removed all the points we can (from the gadget $N$) we
only have to set an initial variable $v_0$ and see if the $q$ chains
of implications (with $q$ being the number of points in $A_{v_0}$)
lead to a solution. So the pattern $T_2$ is tractable.

\paragraph{Proof of tractability of $T_3:$}

Consider an instance from \cspneg{T_3}.

Let $N$ be the following gadget: two variables $v_0$ and $v_1$ such
that we have $a$ in $A_{v_0}$, $b$ and $c$ in $A_{v_1}$, with $b\neq
c$, $a$ compatible with both $b$ and $c$ and $c$ incompatible with a
point in $A_{v_0}$.

Let $d$ be a point in $A_{v_2}$, with $v_2\neq v_0,v_1$. If $d$ is
compatible with $c$ but not with $b$, then we have the forbidden
pattern $T_3$. So if $c$ is compatible with a point outside of
$A_{v_0}$, then $b$ is also compatible with the same point.

Let $S$ be a solution containing $c$. Let $e$ be the point of $S$ in
$A_{v_0}$. If $e$ is compatible with $b$, then we can replace $c$ by
$b$ in $S$ while maintaining the correctness of the solution, since
all the points in the instance outside of $A_{v_0}$ which are
compatible with $c$ are also compatible with $b$.

If $e$ is not compatible with $b$, then edges $\{b,e\}$, $\{e,c\}$
and $\{c,a\}$ form the gadget $N$. So, by our previous argument, if
$e$ is compatible with a point outside of $A_{v_1}$, then $a$ is also
compatible with the same point. We can then replace $c$ by $b$ and
$e$ by $a$ in $S$ while maintaining the correctness of the solution,
since all the points in the instance outside of $A_{v_0}$ which are
compatible with $c$ are also compatible with $b$ and all the points
in the instance outside of $A_{v_1}$ which are compatible with $e$
are also compatible with $a$. So if a solution contains $c$, then
there is another solution containing $b$. Thus we can remove $c$
while preserving solvability.

So each time the gadget $N$ is present, we can remove one of its
points and hence eliminate $N$. The gadget $N$ is a known tractable
pattern since forbidding $N$ is equivalent to saying that all
constraints are either trivial or bijections. So if it is not
present, then the instance is tractable. It follows that the pattern
$T_3$ is tractable.

\paragraph{Proof of tractability of $T_4:$}

Consider an instance from \cspneg{T_4}.

\begin{multicols}{2}

Let $W$ be the following gadget: two variables $v_0$ and $v_1$ such
that we have $a$ in $A_{v_0}$, $b$ and $c$ in $A_{v_1}$, with $b\neq
c$, $a$ compatible with both $b$ and $c$, and $a$ incompatible with a
point in $A_{v_1}$.

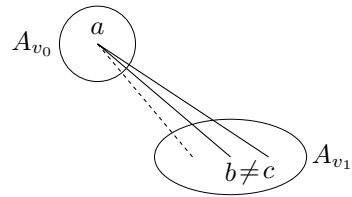
\begin{center}
\scalebox{0.5}{\begin{tikzpicture} \node[scale=2] at (1.5,-0.4)
{$c$}; \node[scale=2] at (0.5,-0.4) {$b$}; \node[scale=2] at (-3,3.4)
{$a$}; \node[scale=2] at (1,-0.4) {$\neq$}; \node[scale=2] at
(-4.7,3) {$A_{v_0}$}; \node[scale=2] at (3.2,0) {$A_{v_1}$};
\draw[thick,color=black,style=dashed] (-0.5,0)--(-3,3);
\draw[thick,color=black] (-3,3)--(0.5,0)
(-3,3)--(1.5,0);
\draw (0.5,0) circle (2 and 1);
\draw (-3,3) circle (1);
\end{tikzpicture}}
\end{center}
\end{multicols}

Let $f$ be a point in $A_{v_2}$, with $v_2\neq v_0,v_1$. If $f$ is
compatible with $b$ but not with $c$, then we have the forbidden
pattern $T_4$. Likewise, if $f$ is compatible with $c$ but not with
$b$, then we have the forbidden pattern $T_4$. So all the points of
the instance not in $A_{v_0}$ or $A_{v_1}$ have the same
compatibility towards $b$ and $c$.

If all points in $A_{v_0}$ compatible with $b$ are also compatible
with $c$, then all the points in the instance compatible with $b$ are
also compatible with $c$ and by neighborhood substitution we can
remove $b$. Thus we can assume there is $d$ in $A_{v_0}$ such that
$d$ is compatible with $b$ but not with $c$.

Let $S$ be a solution containing $c$. Let $e$ be the point of $S$ in
$v_0$. If $e$ is compatible with $b$, then we can replace $c$ by $b$
in $S$ while maintaining the correctness of the solution, since $b$
and $c$ have the same compatibility towards all the points in the
instance outside of $A_{v_0}$ and $A_{v_1}$. If $e$ is not compatible
with $b$, then edges $\{b,e\}$, $\{b,a\}$ and $\{b,d\}$ form the
gadget $W$. So, by our argument above, $a$ and $d$ have the same
compatibility towards all the points in the instance outside of
$A_{v_0}$ and $A_{v_1}$. Similarly, edges $\{c,d\}$, $\{c,a\}$ and
$\{c,e\}$ form the gadget $W$. So $a$ and $e$ have the same
compatibility towards all the points in the instance outside of
$A_{v_0}$ and $A_{v_1}$. So $d$ and $e$ have the same compatibility
towards all the points in the instance outside of $A_{v_0}$ and
$A_{v_1}$. Thus we can replace $c$ by $b$ and $e$ by $d$ in $S$ while
maintaining the correctness of the solution, since $b$ and $c$ have
the same compatibility towards all the points in the instance outside
of $A_{v_0}$ and $A_{v_1}$ and $e$ and $d$ have the same
compatibility towards all the points in the instance outside of
$A_{v_0}$ and $A_{v_1}$. So if a solution contains $c$, then there is
another solution containing $b$. Thus we can remove $c$.

Therefore, each time the gadget $W$ is present, we can remove one of
its points. The gadget $W$ is a known tractable pattern since
forbidding $W$ is equivalent to saying that all constraints are
zero-one-all~\cite{zoa}. So if it is not present, the instance is
tractable. Hence pattern $T_4$ is tractable.

\paragraph{Proof of tractability of $T_5:$}
The pattern $T_5$ is a sub-pattern of the broken-triangle pattern
$BTP$, a known tractable pattern~\cite{btp} on three constraints. So
the pattern $T_5$ is tractable.

\end{proof}

\section{Conclusion}

An avenue for future research is to investigate the possible
generalisations of the five tractable classes defined by forbidding
patterns $T_1,\ldots,T_5$. Possible generalisations include the
addition of costs, replacing binary constraints by $k$-ary
constraints ($k>2$) and adding extra constraints to the patterns.

\end{document}